\documentclass[letterpaper]{article}

\usepackage{epsfig,amssymb,amsmath}
\usepackage{amsfonts}
\usepackage{graphicx,color}
\usepackage{bm}
\usepackage{url}
\usepackage{multirow}
\usepackage{algpseudocode}
\usepackage{algpascal}

\allowdisplaybreaks

\usepackage{times}

\usepackage{amsfonts,mathrsfs,amsthm,amsmath}
\usepackage{multirow}
\usepackage{multicol}
\usepackage{graphicx}
\usepackage{tikz}
\usepackage{caption}
\usepackage{subcaption}
\usepackage[]{algorithm2e}

\usepackage{algpseudocode}
\usetikzlibrary{arrows}

\newtheorem{definition}{Definition}
\newtheorem{proposition}{Proposition}
\newtheorem{theorem}{Theorem}
\newtheorem{lemma}[theorem]{Lemma}

\newtheorem{innermech}{Mechanism}
\newenvironment{mechanism}[1]
{\renewcommand\theinnermech{#1}\innermech}
{\endinnermech}

\usepackage{authblk}

\begin{document}

    \title{Strategyproof Mechanisms for Additively Separable Hedonic Games and Fractional Hedonic Games}
    \author[1]{Michele Flammini}
    \author[2]{Gianpiero Monaco}
    \author[3]{Qiang Zhang}
    \affil[1]{Gran Sasso Science Institute, and University of L'Aquila, Italy\\
     \texttt{michele.flammini@univaq.it}}
    \affil[2]{University of L'Aquila, Italy\\
     \texttt{gianpiero.monaco@univaq.it}}   
    \affil[3]{University of Warsaw, Poland\\
     \texttt{csqzhang@gmail.com}}

%\keywords{Noncooperative games; Strategyproof mechanisms, coalitions formation, additively separable hedonic games, fractional hedonic games.}    

\date{}
\maketitle

\begin{abstract}
Additively separable hedonic games and fractional hedonic games have received considerable attention. 
They are coalition forming games of selfish agents based on their mutual preferences. Most of the work in
the literature characterizes the existence and structure of stable outcomes (i.e.,
partitions in coalitions), assuming that preferences are given.
However, there is little discussion on this assumption. In fact, agents
receive different utilities if they belong to different partitions,
and thus it is natural for them to declare their preferences strategically
in order to maximize their benefit.
%but with little emphasis on the mechanism design. Indeed a major
%challenge today is to design algorithms that work well even when the
%input is reported by selfish agents which aim is maximizing their
%own utility.
In this paper we consider strategyproof mechanisms for additively separable hedonic games and fractional hedonic games, that is, partitioning methods without payments such that utility maximizing agents have no incentive to lie about their
true preferences. We focus on social welfare maximization and provide several lower and upper bounds on the performance achievable by strategyproof mechanisms for general and specific additive functions. In most of the cases we provide tight or asymptotically tight results. All our mechanisms are simple and can be computed in polynomial time. Moreover, all the lower bounds are unconditional, that is, they do not rely on any computational or complexity assumptions. 
\end{abstract}

\section{Introduction}

Teamwork, clustering and group formations, have been important and
widely investigated issues in computer science research. In
many economic, social and political situations, individuals carry
out activities in groups rather than by themselves. In these scenarios,
it is of crucial importance to consider the satisfaction of the members of the groups. 
For example, the utility of an individual in a
group sharing a resource, depends both, on the consumption level of
the resource, and on the identity of the members in the group;
similarly, the utility for a party belonging to a political
coalition depends both, on the party trait, and on the identity of its
members. 

Hedonic games, introduced in \cite{DG1980}, model the formation of coalitions (groups) of players (or agents).  
%When players have preferences over which group they belong to, and describe the dependence
%of an agent's utility on the identity of the members of her group.
They are games in which agents have preferences over the set of all
possible agent coalitions, and the utility
of an agent depends on the composition of the
cluster she belongs to. 

In this paper we consider \emph{additively separable hedonic games} (ASHGs), which constitute a natural and succinctly representable class of hedonic games.   
Each player in an ASHG has a value for any other player, and the utility of a coalition to a particular player
is simply the sum of the values she assigns to the members
of her coalition. Additive separability satisfies a number of
desirable axiomatic properties \cite{ABS11} and ASHGs are the non-transferable utility generalization of
graph games studied by Deng and Papadimitriou \cite{DP94}.
We further consider \emph{fractional hedonic games} (FHGs), introduced in \cite{ABH2014}, which are similar to ASHGs, with the difference that the utility of each agent is divided by the size of her cluster. This allows to model behavioral
dynamics in social environments that are not captured by ASHGs: one usually prefers having a couple of good friends in a
cluster composed by few other people rather than being part of a
crowded cluster populated by uninteresting agents.

Coalition formation in ASHGs and FHGs, has
received growing attention, but mainly from the perspective of coalition
stability, i.e., core, Nash equilibria, etc, or from a classical offline optimization point
of view, i.e., where solutions are not necessarily stable (see Related Work), with
little emphasis on mechanism design. We consider such games where agents have
private preferences. A major challenge is to design algorithms that work well even
when the input is reported by selfish agents aiming only at
maximizing their personal utility. An interesting approach is
to use strategyproof mechanisms \cite{DG2010,PT2009}, that is
designing algorithms (not using payments) where selfish utility
maximizing agents have no incentive to lie about their true
preferences. 

\noindent{\bf Our Contribution.}

We present strategyproof mechanisms for ASHGs and FHGs,
both for general and specific additive valuation functions. In particular, we consider: i) \emph{general valuations} where additive valuations among agents can get any values; ii) \emph{non-negative valuations} where additive valuations among agents can only get positive values; iii) \emph{duplex valuations} where additive valuations among agents can only get values in $\{-1,0,1\}$ (we can think about setting where each agent $i$ can express for any other agent $j$ if she is an enemy, neutral or a friend); iv) \emph{simple valuations} where additive valuations among agents can only get values in $\{0,1\}$ (we can think about setting where each agent $i$ can express for any other agent $j$ if she is neutral or a friend). The latter setting has been also considered in other papers since it models a basic economic scenario referred to in the literature as
Bakers and Millers \cite{ABH2014,BFFMM15}. See Section \ref{Sec:preliminaries} for
more details about the considered valuations.

We focus on the classical utilitarian social welfare, that is the sum of individual utilities of the players in a coalition, and provide several lower and upper bounds on the performance achievable by strategyproof mechanisms.

We are mainly interested in deterministic
mechanisms, however we also provide some randomized lower bounds (notice that randomized lower bounds are stronger than deterministic ones).
Our results are summarized in Table~\ref{table}. In most of the cases (except the case of duplex valuations) we provide tight or asymptotically tight results. 

We point out that, on the one hand, all our
mechanisms are simple and can be computed in polynomial time. On the
other hand, all the lower bounds (some of them randomized) are unconditional, that is, they do not rely on any computational or complexity assumptions.

\noindent{\bf Related Work.}

In the literature, a significant stream of research considered
hedonic games (see \cite{AS16}), and in particular ASHGs, from a strategic cooperative point of view
\cite{BKS01,BJ02,EW09}, with the purpose of characterizing the
existence and the properties of coalitions structures such as the
core, and from a non-cooperative point of view \cite{BD10,FLN15}
with special focus on pure Nash equilibria.
Computational complexity issues related to the problem of computing
stable outcomes have been considered in
\cite{ABS11,GS10,P16,PE2015,W2013}. Finally, hedonic games have also been
considered in \cite{BKKZ13,BBC04,CGW05,DEFI06,DP94} from a classical optimization point of
view, i.e., where solutions are not necessarily stable. Concerning FHGs, Aziz et al. \cite{ABH2014}, give some properties guaranteeing the (non-)existence of the core. Moreover, Brandl et al. \cite{BBS15},
study the computational complexity of understanding the existence of core and individual stable outcomes. From a
non cooperative point of view, the papers \cite{BFFMM14,BFFMM15},
study the existence, efficiency and computational complexity of Nash
equilibria. Other stability notions have
been also investigated, like in \cite{ABH2013,EFF2016}, where the authors
focused on Pareto stability. Finally,
Aziz et al. \cite{AGGMT2015}, consider the computational complexity of computing
welfare maximizing partitions (not necessarily stable).

The design of {\em truthful} mechanisms, that is of algorithms that use payments to convince the selfish agents to
reveal the truth and that then compute the outcome on the basis of their reported
values, has been studied in innumerable scenarios. However, there are
settings where monetary transfers are not feasible, because
of either ethical or legal issues \cite{NRTV2007}, or practical
matters in enforcing and collecting payments \cite{PT2009}. A growing stream of research focuses on the design of the more applicable {\em strategyproof} mechanisms, that lead agents
to report their true preferences, without using payments. %(see \cite{DG2010} and
%all the papers citing it).

Wright et al. \cite{WV2015} focus on strategyproof mechanisms for ASHGs. They only consider positive preferences. Under this assumption, a trivial optimal
strategyproof mechanism just puts all the agents in the same
grand coalition. Therefore, they consider coalition size
constraints and (approximate) envy-freeness. Their main
contribution is a mechanism that, despite not having theoretical
guarantees, achieves good experimental performance. 

Vall{\'{e}}e et al. \cite{VBZB2014} consider classical hedonic games with general preference relationships, and characterize the conditions of the game structure that allow rational false-name manipulations. However, they do not provide mechanisms. Aziz et al. \cite{ABH2013} show that the serial dictatorship mechanism is Pareto optimal, and strategyproof for general hedonic games when appropriate restrictions are imposed on agents. Finally, Rodr{\'{\i}}guez{-}{\'{A}}lvarez \cite{R-A2009}, studies strategyproof core stable solutions properties for hedonic games.

\noindent{\bf Paper organization.}
The paper is organized as follows. In Section \ref{Sec:preliminaries}, we formally describe the problems
and introduce some useful definitions. The studies on the performance of strategyproof mechanisms are then presented in Section \ref{sec:General:valuations_results}, \ref{sec:Non-negative:valuations_results}, \ref{sec:Duplex:valuations_results}, and \ref{sec:Simple:valuations_results}, which address, respectively, general, non-negative, duplex and simple valuations. Finally, in Section \ref{Sec:Conclusion_and_future_work}, we resume our results and list some interesting
open problems.

\begin{table}[h]
 %\centering
    %\resizebox{8.7cm}{!} {
    \label{table}
    \begin{tabular}{cc|c|c|c|c|}
        \cline{3-6}
        &      & {[}-1,1{]}         & {[}0,1{]}        & \{-1,0,1\}                   & \{0,1\}               \\ \hline
        \multicolumn{1}{|c|}{\multirow{2}{*}{ASHGs}}          & L. B.  & \multirow{2}{*}{$Unbounded^*$} & \multirow{2}{*}{OPT}  & $\Omega(n), 2-\epsilon^*$    & \multirow{2}{*}{OPT}              \\ \cline{2-2} \cline{5-5}
        \multicolumn{1}{|c|}{}                                       & U. B.  &                                 &    & $O(n^2)$ &                                                         \\ \hline
        \multicolumn{1}{|l|}{\multirow{2}{*}{FHGs}} & L. B.  & \multirow{2}{*}{$Unbounded^*$} & $\frac{n}{2}$   & $2-\epsilon$                             & $\frac{6}{5}$           \\ \cline{2-2} \cline{4-6}
        \multicolumn{1}{|l|}{}                                       & U. B.  &                                & $\frac{n}{2}$  & $O(n)$                          &$2$                                \\ \hline
    \end{tabular}
    %}
    \caption{Our results for the different cases.
    * stands for randomized mechanisms. L. B. stands for lower bounds. U. B. stands for upper bounds.}
    \label{table}
\end{table}

\iffalse

\begin{table}[h]
 %\centering
    \resizebox{8.7cm}{!} {
    \label{table}
    \begin{tabular}{cc|c|c|c|c|}
        \cline{3-6}
        &             & {[}0,1{]}            & \{0,1\}              & {[}-1,1{]}                 & \{-1,0,1\} \\ \hline
        \multicolumn{1}{|c|}{\multirow{2}{*}{Additively Separable Hedonic Games}}          & Lower bound & \multirow{2}{*}{OPT} & \multirow{2}{*}{OPT} & \multirow{2}{*}{$Unbounded^*$} & $\Omega(n), 2-\epsilon^*$            \\ \cline{2-2} \cline{6-6}
        \multicolumn{1}{|c|}{}                                       & Upper bound &                      &                      &                            & $O(\frac{n}{\log n})^*$           \\ \hline
        \multicolumn{1}{|l|}{\multirow{2}{*}{Fractional Hedonic Games}} & Lower bound & $\frac{n}{2}-\epsilon$          & $\frac{6}{5}$                     & \multirow{2}{*}{$Unbounded^*$} & $2-\epsilon$            \\ \cline{2-4} \cline{6-6}
        \multicolumn{1}{|l|}{}                                       & Upper bound & $\frac{n}{2}$                 &$2$                      &                            & $O(\frac{n}{\log n})^*$           \\ \hline
    \end{tabular}
    }
    \caption{Our results for the different cases.
    * stands for randomized mechanisms.}
    \label{table}
\end{table}

\fi

\section{Preliminaries}\label{Sec:preliminaries}
In additive separable hedonic games (ASHGs) and fractional hedonic games (FHGs), we are given a set
$N=\{1,\ldots, n\}$ of selfish agents. The objective or outcome of the game is a partition of
the agents into disjoint coalitions $\mathcal{C} = \{C_1,
C_2,\ldots\}$, where each coalition $C_j$ is a subset of agents and
each agent is in exactly one coalition. Let $\mathscr{C}$ be the collection of all
the possible outcomes. Given a partition $\mathcal{C} \in \mathscr{C} $, we denote by $|\mathcal{C}|$  the number of its
coalitions and by $\mathcal{C}^i$ the
coalition of $\mathcal{C}$ containing agent $i$. Similarly, given a
coalition $C$, we let $|C|$ be the size or number of agents in $C$.
%Given a set of disjoint coalitions $\mathcal{C}$, let $\mathcal{C}^i$ be the
%coalition that agent $i$ is a member of.
The \textit{grand coalition} is the outcome in which all the agents
are in the same coalition, i.e., $|\mathcal{C}|=1$. We assume that
each agent has a privately known valuation $v_i: N \rightarrow
\mathbb{R}$, mapping every agent to a real (possibly negative) value.
In ASHGs, for any $\mathcal{C} \in \mathscr{C}$, the
preference or utility of agent $i$ is
$u_i(\mathcal{C}) = \sum_{j \in \mathcal{C}^i} v_{i}(j)$, that is, it
is additively induced by her valuation function.
%\[
%u_i(\mathcal{C}) = \sum_{j \in \mathcal{C}^i} v_{i}(j)
%\]
Similarly, in FHGs,  for any $\mathcal{C} \in
\mathscr{C}$, the utility of agent $i$ is $u_i(\mathcal{C}) =
\frac{\sum_{j \in \mathcal{C}^i} v_{i}(j)}{|\mathcal{C}^i|}$.

%\[
%u_i(\mathcal{C}) = \frac{\sum_{j \in \mathcal{C}^i}
%    v_{i}(j)}{|\mathcal{C}^i|}
%\]

We are interested in four basic classes of valuation functions. Namely, for any pair of agents $i,j \in N$, we consider:
{\em General valuations}: $v_{i}(j) \in [-1, 1]$;
{\em Non-negative valuations}: $v_{i}(j) \in [0, 1]$;
{\em Duplex valuations}: $v_{i}(j) \in \{-1, 0, 1\}$;
{\em Simple valuations}: $v_{i}(j) \in \{0, 1\}$.
In every case, we assume that $v_{i}(i) = 0$, for every $i \in N$. Notice that any valuation function can be represented by using values in the range $[-1, 1]$.

Agents are self-interested entities. Thus, they may strategically
misreport their valuation functions in order to maximize their
utilities. %In this paper, we are interested in strategyproof mechanisms to cope with
%the agents' strategic behavior. 
Let $\mathbf{d}$ denote the
preferences (valuation functions) declared by all the agents. 

A deterministic mechanism
$\mathcal{M}$ maps every set (or list) of preferences $\mathbf{d}$ to a set of disjoint coalitions $\mathcal{M}(\mathbf{d}) \in
\mathscr{C}$.
%Let $|\mathcal{C}|$ denote the number of coalitions in $\mathcal{C}$, and let $|\mathcal{C}_j|$ denote the number of agents in coalitions $\mathcal{C}_j$.
We denote by $\mathcal{M}^i(\mathbf{d})$ the coalition assigned
to agent $i$ by $\mathcal{M}$. The utility of agent $i$ is given
by $u_i(\mathcal{M}(\mathbf{d}))$.
%For any pair of two
%coalitions $C$ and $C' \in \mathcal{M}(\mathbf{d})$, it requires
%that $C \cap C' = \emptyset$. In addition, it requires that $\sum_{C
%\in \mathcal{M}(\mathbf{d})} |C| = n$.
%In this paper we are mainly
%interested in deterministic mechanism. However we allow our
%mechanisms to be randomized.
Let $\mathbf{d}_{-i}$ be the valuation functions declared by all agents except agent
$i$ and $d_i$ be a possible declaration of valuation function by $i$. A deterministic mechanism $\mathcal{M}$ is \textit{strategyproof} if
for any $i \in N$, any list of preferences $\mathbf{d}_{-i}$, any $v_i$ and
any $d_i$, it holds that $u_i(\mathcal{M}(\mathbf{d}_{-i}, v_i))
\geq u_i(\mathcal{M}(\mathbf{d}_{-i}, d_i))$. In other words, a strategyproof mechanism prevents any agent $i$ from benefiting by declaring a valuation different from $v_i$, whatever the other declared valuations are.

A randomized mechanism $\mathcal{M}$ maps every set of agents' preferences
$\mathbf{d}$ to a distribution $\Delta$ over the set of all the possible outcomes $\mathscr{C}$. The expected utility of agent $i$ is given by $ \mathbb{E}[u_i(\mathcal{M}(d))] =
\mathbb{E}_{\mathcal{C} \sim \Delta}[u_i(\mathcal{C})]$.
%A deterministic mechanism $\mathcal{M}$ is \textit{acceptable} if for any preferences $\mathbf{v}$ the coalition obtained by mechanism $\mathcal{M}$ is acceptable for every agent, that is, for any $\mathbf{v}$ and any $i \in N$, it holds that $u_i(\mathcal{M}^i(\mathbf{v})) \geq 0$. A randomized mechanism $\mathcal{M}$ is is \textit{acceptable} if for any preference
%A deterministic mechanism $\mathcal{M}$ is \textit{acceptable} if for any preferences $\mathbf{v}$ the set of disjoint coalitions obtained by mechanism $\mathcal{M}$ is acceptable, that is, for any $\mathbf{v}$ it holds that $\sum_{i \in N}u_i(\mathcal{M}^i(\mathbf{v})) \geq 0$.
A randomized mechanism $\mathcal{M}$ is strategyproof (in expectation)
if for any $i \in N$, any preferences $\mathbf{d}_{-i}$, any $v_i$
and any $d_i$, $\mathbb{E}[u_i(\mathcal{M}(\mathbf{d}_{-i}, v_i))] \geq
\mathbb{E}[u_i(\mathcal{M}(\mathbf{d}_{-i}, d_i))]$. 

%Notice that
%strategyproof randomized mechanisms only
%motivate risk-neutral agents to act truthfully, while risk-averse
%ones may still benefit from strategic behavior.

In this paper, we are interested in strategyproof
mechanisms that perform well with respect to the goal of maximizing the
classical utilitarian social welfare, that is, the sum of the utilities achieved by all the agents. Namely,
the social welfare of a given outcome $\mathcal{C}$
is $SW(\mathcal{C})= \sum_{i \in N} u_i(\mathcal{C})$.
We denote by $SW(C)=\sum_{i \in C} u_i(\mathcal{C})$ the overall
social welfare achieved by the agents belonging to a given coalition $C$. We
measure the performance of a mechanism by comparing the social
welfare it achieves with the optimal one. More precisely, the
approximation ratio of a deterministic mechanism $\mathcal{M}$ is
defined as $r^{\mathcal{M}} = \sup_{\mathbf{d}} \frac{
\mathtt{OPT}(\mathbf{d}) }{SW(\mathcal{M}(\mathbf{d}))}$, where $\mathtt{OPT}(\mathbf{d})$
is the social welfare achieved by an optimal set of coalitions in the
instance induced by $\mathbf{d}$.
%\[
%r^{\mathcal{M}} = \sup_{\mathbf{d}} \frac{ \mathtt{OPT}(\mathbf{d})
%}{\sum_{i \in N} u_i(\mathcal{M}^i(\mathbf{d}))  }
%\]
For randomized mechanisms, the
approximation ratio is computed with respect to the expected
social welfare, that is
$r^{\mathcal{M}} = \sup_{\mathbf{d}}
\frac{ \mathtt{OPT}(\mathbf{d})}{\mathbb{E} [
SW(\mathcal{M}(\mathbf{d}))] }$.

We say that a deterministic mechanism $\mathcal{M}$ is
\textit{acceptable} if it always guarantees a non negative social
welfare, i.e., $SW(\mathcal{M}(\mathbf{d})) \geq 0$ for any
possible list of preferences $\mathbf{d}$. Similarly, a randomized
mechanism $\mathcal{M}$ is \textit{acceptable} if $\mathbb{E} [
SW(\mathcal{M}(\mathbf{d}))] \geq 0$ holds for every $\mathbf{d}$. In the following, we will always implicitly restrict to acceptable
mechanisms. In fact, a simple acceptable strategyproof mechanism for
all the considered classes of valuations can be trivially obtained
by putting every agent in a separate singleton coalition, regardless
of all the declared valuations.

 %An acceptable mechanism requires that every agent $i$ weakly prefers $\mathcal{M}^i(\mathbf{v})$ to being alone regardless the preferences $\mathbf{v}$.

% (see \cite{DNS2012} for a discussion).

%A randomized mechanism $\mathcal{M}$ is \textit{universally
%strategyproof} if it is a probability distribution over
%strategyproof deterministic mechanisms.

%\[
%r^{\mathcal{M}} = \sup_{\mathbf{d}} \frac{
%\mathtt{OPT}(\mathbf{d})}{\sum_{i \in N}\mathbb{E} [ u_i(\mathcal{M}(\mathbf{d}))]  }
%\]
%where $\mathcal D$ is the probability distribution of the mechanism outputs.

%\paragraph{Graph representation}
\noindent{\bf Graph representation.} 
ASHGs and FHGs have a
 very intuitive graph representation. In fact, any instance of the games can be expressed by a
 weighted directed graph $G=(V,E)$, where nodes in $V$ represent the agents, and arcs or directed edges are associated to non null valuations. Namely, if $v_i(j) \neq 0$, an arc $(i,j)$ is contained in $E$ of weight  $w(i,j) = v_i(j)$.
As an example, in case of simple valuations, if $(i,j) \notin E$ then $v_{i}(j)=0$, while if $(i,j) \in E$ then $w(i,j) = v_i(j)=1$.

%For any arc $(i,i') \in E$ we denote by $v_{ii'}$ the
%preference of agent $i$ for agent $i'$, i.e., $v_{ii'} = v_i(i')$.
%We assume that if $(i,i') \notin E$ then $v_{ii'}=0$. Let
%$\mathbf{v}=\{v_1,\ldots,v_n\}$ denote the preferences of all
%agents. Notice that $\mathbf{v}$ directly induces the graph
%$G=(N,E,v)$.

Throughout the paper we will sometimes describe an instance of the considered game by its graph representation. In the following sections, we provide our results for all of the four considered classes of valuation functions.

\section{General valuations}\label{sec:General:valuations_results}
In this section, we consider the setting where agents have general valuations. 
We are able to prove that there is no randomized strategyproof
mechanism with bounded approximation ratio both for ASHGs and FHGs. %even if we do not consider computational constraints.
Clearly, the theorem applies also to deterministic mechanisms, since they are special cases of randomized ones.

\begin{theorem}\label{LB:General_and_randomized:HGandFHG}
For general valuation functions, there is no randomized strategyproof acceptable mechanism with bounded approximation ratio both for ASHGs and FHGs.
\end{theorem}

\begin{proof}
%We use the same instances for both the two classes of games. 
We prove the claim only for ASHGs. However,
the same arguments directly apply also to FHGs.

\begin{figure}[h]%
        \centering
        %\subfigure
        \begin{subfigure}[b]{0.40\textwidth}
            \begin{tikzpicture}[scale=0.06,->,>=stealth',shorten >=1pt,auto,node distance=1.3cm,
            thick,main node/.style={circle,draw,font=\sffamily\Large\bfseries}]
            %\tikzset{edge/.style = {->,> = latex'}};
            \node[main node] (1) {1};
            \node[main node] (2) [right of=1] {2};
            \node[main node] (3) [right of=2] {3};
            \path[every node/.style={font=\sffamily\small}]
            (1) edge [bend right] node  {$\epsilon$} (2)
            (2) edge [bend right]  node {$-1$} (3)
            %(2) edge [bend right] node[above] {$2\epsilon$} (1)
            (3) edge [bend right]  node[above] {$0.9$} (2);
            \end{tikzpicture}
            \caption{Instance $I_1$}
            \label{fig:LB_general_valuations_a}
        \end{subfigure}
        \begin{subfigure}[b]{0.40\textwidth}
            \begin{tikzpicture}[scale=0.06,->,>=stealth',shorten >=1pt,auto,node distance=1.3cm,
            thick,main node/.style={circle,draw,font=\sffamily\Large\bfseries}]
            %\tikzset{edge/.style = {->,> = latex'}};
            \node[main node] (1) {1};
            \node[main node] (2) [right of=1] {2};
            \node[main node] (3) [right of=2] {3};
            \path[every node/.style={font=\sffamily\small}]
            (1) edge [bend right] node  {$\epsilon$} (2)
            (2) edge [bend right]  node {$-\epsilon$} (3)
            %(2) edge [bend right] node[above] {$2\epsilon$} (1)
            (3) edge [bend right]  node[above] {$0.9$} (2);
            \end{tikzpicture}
            \caption{Instance $I_2$}
            \label{fig:LB_general_valuations_b}
        \end{subfigure}
        \caption{The lower bound instance for general valuations.}
        \label{fig:LB_general_valuations}
    \end{figure}
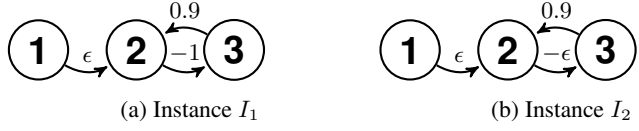

Let $\mathcal{M}$ be a given randomized strategyproof mechanism.
Provided that $\mathcal{M}$ is strategyproof, we implicitly assume
that the agents' declared preferences
$\mathbf{d}$ correspond to the true valuation functions. Let us then
consider the instance $I_1$ depicted in
Figure~\ref{fig:LB_general_valuations_a}, and let $p$ be the
probability that $\mathcal{M}$ returns an outcome for $I_1$ where
agents $2$ and $3$ are together in the same coalition.
%Let us call $rm$ the outcome of the randomized mechanism.
Then, the expected social welfare is $E[SW(\mathcal{M}(\mathbf{d}))]
\leq p(\epsilon - 0.1)+(1-p)\epsilon=\epsilon - 0.1p$, while the
optimal solution has social welfare $\epsilon$. Therefore, the
randomized mechanism has bounded approximation ratio only if
$\epsilon - 0.1p >0$, that implies $p<10\epsilon$. Let us now
consider the instance $I_2$ depicted in
Figure~\ref{fig:LB_general_valuations_b}, and let $q$ be the
probability that mechanism $\mathcal{M}$ returns an outcome where
agents $2$ and $3$ are together in the same coalition. Then the
expected social welfare is $E[SW(\mathcal{M}(\mathbf{d}))] \leq 0.9q
+ (1-q)\epsilon$. We notice that $\mathcal{M}$ can be strategyproof
only if $p \geq q$, otherwise agent $2$ could improve her utility by
declaring value $-1$ for agent $3$, since in such a case she would
get utility $-p\epsilon > -q\epsilon$. The optimal solution of
instance $I_2$ has value $0.9$. Thus, the approximation ratio of
$\mathcal{M}$ is
$\frac{\mathtt{OPT}(\mathbf{d})}{E[SW(\mathcal{M}(\mathbf{d}))] }
\geq \frac{0.9}{0.9q + (1-q)\epsilon} \geq \frac{0.9}{0.9q +
\epsilon} \geq \frac{0.9}{10\epsilon}$. As $\epsilon$ can be
arbitrarily small, we can then conclude that $\mathcal{M}$ has an
unbounded approximation ratio. The claim then follows by the
arbitrariness of $\mathcal{M}$.
\end{proof}

\section{Non-negative valuations}\label{sec:Non-negative:valuations_results}
In this section, we consider the setting where agents have non-negative valuations. 
Let us first present a simple optimal mechanism for non-negative valuations in ASHGs.

\begin{mechanism}{$\mathcal{M}_1$}\label{mech:grandCoalition}
    Given as input a list of agents' valuations $\mathbf{d}=\langle d_1,...,d_n\rangle$, the mechanism outputs the grand coalition, i.e. $\mathcal{M}(\mathbf{d}) = \{ \{1,\ldots,n\} \}$.
\end{mechanism}

It is trivial to see that, in ASHGs with non-negative
valuations, the above mechanism $\mathcal{M}_1$ is acceptable,
strategyproof, and achieves the optimal social welfare. Therefore,
we now focus on FHGs. We are able to show that any
deterministic strategyproof mechanism cannot have an approximation
better than $\frac{n}{2}$. 
%even if we do not consider computational constraints.

\begin{theorem}\label{LB:NonNegative:FHG}
For FHGs with non-negative valuations, no
deterministic strategyproof acceptable mechanism can achieve approximation ratio $r$, with $r < \frac{n}{2}$.
\end{theorem}

\begin{proof}

\begin{figure}[h]%
    \centering
    %\subfigure
    \begin{subfigure}[b]{0.40\textwidth}
        \begin{tikzpicture}[->,>=stealth',shorten >=1pt,auto,node distance=1.5cm,
        thick,main node/.style={circle,draw,font=\sffamily\Large\bfseries}]
        %\tikzset{edge/.style = {->,> = latex'}};
        \node[main node] (1) {1};
        \node[main node] (2) [below right of=1] {2};
        \node[main node] (3) [below left of=2] {3};
        \node[main node] (4) [above left of=3] {4};
        \path[every node/.style={font=\sffamily\small}]
        (1) edge node {$\alpha$} (2)
        (2) edge node {$\beta$} (3)
        (3) edge node {$\alpha$} (4)
        (4) edge node {$\beta$} (1);
        \end{tikzpicture}
        \caption{Instance $I_1$}
        \label{fig:LB_Non-negative_valuations_a}
    \end{subfigure}
        \begin{subfigure}[b]{0.40\textwidth}
            \begin{tikzpicture}[->,>=stealth',shorten >=1pt,auto,node distance=1.5cm,
            thick,main node/.style={circle,draw,font=\sffamily\Large\bfseries}]
            %\tikzset{edge/.style = {->,> = latex'}};
            \node[main node] (1) {1};
            \node[main node] (2) [below right of=1] {2};
            \node[main node] (3) [below left of=2] {3};
            \node[main node] (4) [above left of=3] {4};
            \path[every node/.style={font=\sffamily\small}]
            (1) edge node {$\alpha$} (2)
            (2) edge node {$\beta$} (3)
            (3) edge node {$\alpha$} (4)
            (4) edge node {$1$} (1);
            \end{tikzpicture}
            \caption{Instance $I_2$}
            \label{fig:LB_Non-negative_valuations_b}
        \end{subfigure}
    \caption{The lower bound instances for non-negative valuations with $4$ agents.}
    \label{fig:LB_Non-negative_valuations}
\end{figure}
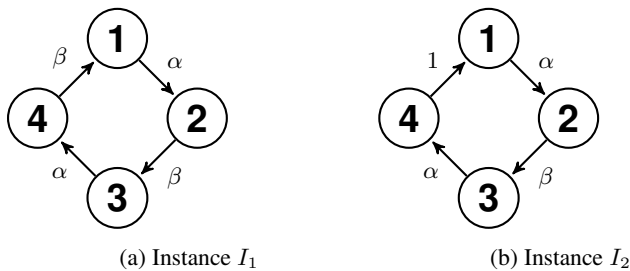

Assume $\frac{1}{n} \gg \alpha \gg \beta$.  Let us consider the instance $I_1$
with an even number $n$ of agents, where the valuation functions are
as follows:
\begin{itemize}
    \item for any $i= 1, 3, \ldots, n - 1$, $v_i(j) = \alpha$ if $j=i+1$ and $v_i(j)=0$ otherwise;
    \item for any $i=2, 4, \ldots, n-2$, $v_i(j) = \beta$ if $j=i+1$ and $v_i(j)=0$ otherwise;
    \item $v_n(1) = \beta$ and $v_n(j) = 0 $ for any $j\neq 1$.
\end{itemize}
%Figure~\ref{fig:LB_Non-negative_valuations_a} is such instance with
%$n=4$ agents.
The optimal outcome is given by the set of coalitions  
$\mathcal{C} = \{C_1,C_2,$ \\
$\ldots,C_{\frac{n}{2}}\}$, where $C_j=\{2j-1,2j\}$ for
any $j=1,\ldots,\frac{n}{2}$, and achieves social welfare
$\frac{n}{4}\alpha$. We now show that any deterministic
strategyproof mechanism with an approximation ratio lower than
$\frac{n}{2}$ has to output the grand coalition. In fact, the grand
coalition has social welfare $\frac{\alpha+\beta}{2}$, which has
approximation ratio tending to $\frac{n}{2}$ when $\beta/\alpha$
tends to $0$, thus proving the claim. Assume then that a
deterministic strategyproof mechanism $\mathcal{M}$ with
approximation ratio strictly less than $\frac{n}{2}$ outputs an
outcome different from the grand coalition. In this case, there
must be at least one agent $k$ having null utility. But then $k$ might improve her utility by declaring $v_k(k+1)=1$, %Figure~\ref{fig:LB_Non-negative_valuations_b} is such
%instance when agent $4$ declares $v_4(1)=1$.
as in this case $\mathcal{M}$, since $\alpha \ll \frac{1}{n}$, in order to achieve approximation less than $\frac{n}{2}$ must give an outcome in which agents $k$ and $k+1$
are in the same coalition. Hence, agent $k$ improves her utility by
declaring $v_k(k+1)=1$. Therefore $\mathcal{M}$ for the instance $I_1$ has to output the grand coalition, thus proving the theorem.
\end{proof}

Given the above result, it is easy to show that, returning the grand coalition is the best we can do. 

\begin{proposition}\label{UB:NonNegative:FHG}
For FHGs with non-negative valuations,
Mechanism $\mathcal{M}_1$  is a deterministic strategyproof acceptable
 mechanism with approximation ratio $\frac{n}{2}$.
\end{proposition}

\begin{proof}
As valuations are non-negative and Mechanism $\mathcal{M}_1$ always
outputs the grand coalition, the mechanism is clearly acceptable and
strategyproof. Let us now focus on its approximation ratio for the
social welfare. Notice that, given any $\mathbf{d}$, then
$\mathtt{OPT}(\mathbf{d}) \leq \frac{\sum_{i \in N} \sum_{j \in N}
v_i(j)}{2}$. This is because any coalition in the optimal coalitions
with positive social welfare consists of at least two agents.
Otherwise, the coalition has zero social welfare since $v_i(i)=0$
for any $i \in N$. On the other hand the grand coalition has social
welfare equal to $\frac{\sum_{i \in N} \sum_{j \in N:} v_i(j)}{n}$.
The approximation ratio follows.
\end{proof}

\section{Duplex valuations}\label{sec:Duplex:valuations_results}
In this section, we consider the setting where agents have duplex valuations. 
We first present deterministic lower bounds for ASHGs and FHGs. 
%without assuming any computational constraint.

\begin{theorem}\label{LB:Duplex:HG}
For ASHGs with duplex valuations, no deterministic strategyproof acceptable
 mechanism has approximation ratio less than $n-2$.
\end{theorem}

\begin{proof}
Let us consider the instance $I_1$ depicted in
Figure~\ref{fig:LB_simple_valuations_a}, where the valuations of the $n$
agents are as follows: \\
- for $i=1, \ldots, n-2$, $v_i(j) = 1$ if $j=n-1$ and $v_i(j) = 0$ otherwise;\\
- $v_{n-1}(j) = 1$ if $j = n$ and $v_{n-1}(j) = -1$ otherwise; \\
- $v_n(j) = -1$ for $j = 1, \ldots, n-2$ and $v_n(n-1)=0$. \\

%\begin{itemize}
%\item  for $i=1, \ldots, n-2$, $v_i(j) = 1$ if $j=n-1$ and $v_i(j) = 0$ otherwise;
%\item $v_{n-1}(j) = 1$ if $j = n$ and $v_{n-1}(j) = -1$ otherwise;
%\item $v_n(j) = -1$ for $j = 1, \ldots, n-2$ and $v_n(n-1)=0$.
%\end{itemize}

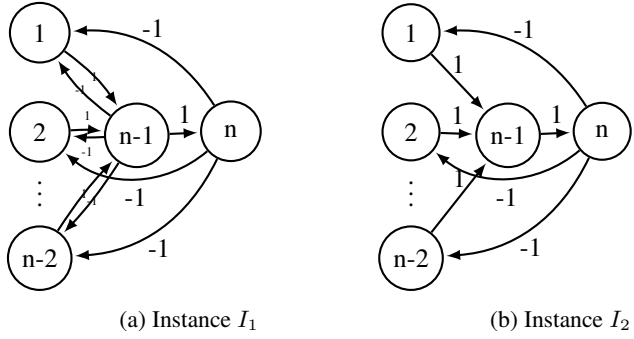
\begin{figure}[h]
\centering
    %\subfigure
    \begin{subfigure}[b]{0.40\textwidth}
     \begin{tikzpicture}[scale=0.12,->,>=stealth',shorten >=1pt,auto,node distance=1.1cm,
     thick,main node/.style={circle,draw,minimum width={width("assas")+2pt}}]
     \node[main node] (n-2) at (90:5) {n-2};
     \node[main node] (2) at (90:19) {2};
     \node[main node] (1) at (90:30) {1};
     \node[main node] (n-1) at (60:21.5) {n-1};
     \node[main node] (n) at (42:28.5) {n};
     \node[] at (90:13) {$\vdots$};
     \draw[-latex] (1) to[bend left=10] node[anchor=center] {\tiny 1} (n-1);
     \draw[-latex] (n-1) to[bend left=10] node[anchor=center] {\tiny -1} (1);
     \draw[-latex] (2) to[bend left=5] node[above] {\tiny 1} (n-1);
     \draw[-latex] (n-1) to[bend left=5] node[below] {\tiny -1} (2);
     \draw[-latex] (n-2) to[bend left=5] node[anchor=center] {\tiny 1} (n-1);
     \draw[-latex] (n-1) to[bend left=5] node[anchor=center] {\tiny -1} (n-2);
     \draw[-latex] (n-1) to node[above] {1} (n);
     \draw[-latex] (n) to[bend right=30] node[above] {-1} (1);
     \draw[-latex] (n) to[bend left=40] node[below] {-1} (2);
     \draw[-latex] (n) to[bend left=30] node[below] {-1} (n-2);
     \end{tikzpicture}
    \caption{Instance $I_1$}
        \label{fig:LB_simple_valuations_a}
    \end{subfigure}
    \begin{subfigure}[b]{0.40\textwidth}
    \begin{tikzpicture}[scale=0.12,->,>=stealth',shorten >=1pt,auto,node distance=1.1cm,
    thick,main node/.style={circle,draw,minimum width={width("assas")+1pt}}]
    \node[main node] (n-2) at (90:5) {n-2};
    \node[main node] (2) at (90:19) {2};
    \node[main node] (1) at (90:30) {1};
    \node[main node] (n-1) at (60:21.5) {n-1};
    \node[main node] (n) at (42:28.5) {n};
    \node[] at (90:13) {$\vdots$};
    \draw[-latex] (1) to node[above] {1} (n-1);
    \draw[-latex] (2) to node[above] {1} (n-1);
    \draw[-latex] (n-2) to node[above,midway] {1} (n-1);
    \draw[-latex] (n-1) to node[above] {1} (n);
    \draw[-latex] (n) to[bend right=30] node[above] {-1} (1);
    \draw[-latex] (n) to[bend left=40] node[below] {-1} (2);
    \draw[-latex] (n) to[bend left=30] node[below] {-1} (n-2);
    \end{tikzpicture}
     \caption{Instance $I_2$}
        \label{fig:LB_simple_valuations_b}
    \end{subfigure}
     \caption{The lower bound instance for duplex valuations.}
    \label{fig:LB_simple_valuations}
\end{figure}

In the optimal outcome agents $n-1$ and $n$ are in the same
coalition and all other agents are in different coalitions. The
resulting social welfare is $1$, and in particular it is due to
agent $n-1$ having utility $1$. It is easy to see that any mechanism
having bounded approximation has to return the optimal outcome, as
any other solution would have social welfare at most zero. Let us
now consider the other instance $I_2$ depicted in
Figure~\ref{fig:LB_simple_valuations_b}, where agent $n-1$ is the
only one with a different valuation function with respect to $I_1$,
that is $v_{n-1}(n) = 1$ and $v_{n-1}(j) = 0$ for $j\neq n$. Any
strategyproof mechanism with bounded approximation ratio for $I_2$
has to put agents $n-1$ and $n$ in the same coalition, otherwise
$n-1$ would have null utility and could increase her utility by
declaring her valuation function as it is in instance $I_1$.
Moreover, any outcome in which $n-1$ and $n$ are together,
independently from the other coalitions, has social welfare $1$.
However, the optimal outcome, by putting $1,2,\ldots,n-1$ all
together in a same coalition and agent $n$ alone, achieves social
welfare $n-2$. This proves the $n-2$ lower bound for any
deterministic strategyproof mechanism.
\end{proof}

\begin{theorem}\label{LB:Duplex:FHG}
For FHGs with duplex valuations, no deterministic strategyproof acceptable mechanism can achieve approximation $2 - \epsilon$,
for any $\epsilon>0$.
\end{theorem}

\begin{proof}
The proof is very similar to Theorem~\ref{LB:Duplex:HG}, but here
the optimal solution has value $\frac{n-2}{n-1}$ and the best
strategyproof acceptable mechanism returns an outcome of social
welfare $\frac{1}{2}$. It follows that for big value of $n$, the ratio tends to $2$,
and thus proving the theorem.
\end{proof}

We are also able to prove the following randomized lower bound.

\begin{theorem}\label{LB:Duplex_and_randomized:HG}
For ASHGs with duplex valuations, no randomized
strategyproof acceptable mechanism can achieve approximation $2 -
\epsilon$, for any $\epsilon>0$.
\end{theorem}

\begin{proof}
Let us consider the instance $I_1$ depicted in
Figure~\ref{fig:LB_simple_valuations_a}. Let $p$ be the probability
that a randomized mechanism returns the outcome where agents $n-1$
and $n$ are together in the same coalition and all the other agents
are alone. Notice that in such a case agent $n-1$ has expected
utility equal to $p$. Let us call $rm$ the outcome of the randomized
mechanism. Then the expected social welfare in this case is such
that $\mathbb{E}[rm] \leq p$. Let us now consider the instance $I_2$
depicted in Figure~\ref{fig:LB_simple_valuations_b}. Let $q$ be the
probability that a randomized mechanism returns an outcome where
agents $n-1$ and $n$ are together in the same coalition (possibly
with other agents). Notice that the social welfare of any outcome
where agents $n-1$ and $n$ are together is always $1$, independently
from the coalitions of the other agents are member of. Moreover,
notice that in such a case agent $n-1$ has expected utility equal to
$q$. On the other hand, the mechanism with probability $1-q$ put
agents $n-1$ and $n$ not together in the same coalition. In such a
case, i.e., with probability $1-q$, the social welfare is at most
equal to $n-2$. Let us call $rm'$ the outcome of the randomized
mechanism. It turns out that the expected social welfare in this case is such
that $\mathbb{E}[rm'] \leq q + (1-q)(n-2)$. We notice that such
mechanism is strategyproof only if $q \geq p$. In fact, if $p>q$, then
agent $n-1$ can improve her utility by declaring value
$v_{n-1}(j)=-1$, for any $j=1,\ldots,n-2$, and $v_{n-1}(n)=1$ (thus
reconstructing the instance $I_1$), since in such a case she would
get expected utility $p > q$. Therefore, the expected social welfare
of the mechanism of $I_1$ is maximized when $p=q$. We now equalize
the expected approximation ratio of the mechanisms of both instances
(where we set $p=q$), where $1$ is the optimal value for the instance depicted in Figure~\ref{fig:LB_simple_valuations_a}, and $n-2$ is the optimal value for the instance depicted in Figure~\ref{fig:LB_simple_valuations_b}. 

$\frac {1}{\mathbb{E}[rm]}= \frac {n-2}{\mathbb{E}[rm']} \implies  
\frac {1}{q}=\frac{n-2}{q + (1-q)(n-2)} \implies q= \frac{q + (1-q)(n-2)}{n-2}  \implies  
\frac{(q-1)(n-2)}{q(3-n)} = 1 \implies q=\frac{n-2}{2n-5}.$

%\begin{align}
%& \frac {1}{\mathbb{E}[rm]}= \frac {n-2}{\mathbb{E}[rm']} \implies  
%& \frac {1}{q}=\frac{n-2}{q + (1-q)(n-2)} \implies  \nonumber \\
%& q= \frac{q + (1-q)(n-2)}{n-2}  \implies  
%& \frac{(q-1)(n-2)}{q(3-n)} = 1 \implies \nonumber \\
%& q=\frac{n-2}{2n-5}. \nonumber
%\end{align}

It follows that for big value of $n$, $q$ tends to $\frac{1}{2}$,
and thus proving the theorem.
\end{proof}

We now present a deterministic strategyproof acceptable mechanism \ref{mech:duplex_HG_and FHG} with approximation $O(n^2)$ for ASHGs and $O(n)$ for FHGs. We doubt the existence of deterministic strategyproof acceptable mechanisms with approximation ratio $O(n)$ for ASHGs and $O(1)$ for FHGs. We provide some discussion supporting it, at the end of the section. Closing the gap for duplex valuations, is one of the main open problem. 

The following definition is crucial for Mechanism \ref{mech:duplex_HG_and FHG}. 
\begin{definition}
Given $\mathbf{d}=\langle d_1,...,d_n\rangle$ declared by the set of agents $N$, we say that an agent $i \in N$ is a \emph{sink} if there is no agent $j \in N$ such that $d_i(j)=1$ and $d_j(i) \neq -1$. 
%That is agent $i$ is a sink if she has no chance to provide a positive contribution to the social welfare. 
\end{definition}

The idea of the mechanism is as follows. It considers the agents in an arbitrary ordering. If the considered agent $i$ has value $1$ for some other agent $j$, such that $j$ also has value $1$ for $i$, or $j$ is a sink, or $j$ is before $i$ in the ordering, then it returns agents $i$ and $j$ together in a coalition, and each other agent in a coalition alone. If, after considering all the agents, the mechanism does not create the coalition with two agents, then returns each agent in a coalition alone.  
It follows the formal description of the mechanism \ref{mech:duplex_HG_and FHG}.

\begin{mechanism}{$\mathcal{M}_2$}\label{mech:duplex_HG_and FHG}
Given any declared valuation $\mathbf{d}=\langle d_1,...,d_n\rangle$, the mechanism performs as follows: \\
1 Consider any ordering of the agents and, for the sake of simplicity, let $i$ be the $i$-th agent in such ordering. \\
2 For $i=1$ to $n$:
\begin{itemize}
    \item[a] If there exists $j \in N$ such that $d_i(j) = 1$ $\wedge$ $d_j(i)=1$: put agents $i$ and $j$ together into a coalition and any other agent alone, and terminate.
    \item[b] If there exists $j \in N$ such that $d_i(j) = 1$ $\wedge$ $d_j(i)=0$ $\wedge$ $j$ is a sink: put agents $i$ and $j$ together into a coalition and any other agent alone, and terminate.
    \item[c] If there exists $j \in N$ such that $d_i(j) = 1$ $\wedge$ $d_j(i)=0$ $\wedge$ $j<i$: put agents $i$ and $j$ together into a coalition and any other agent alone, and terminate.
\end{itemize}
3 If no coalition of two agents has been created during the step 2: return each agent in a coalition on its own. 
\end{mechanism}

%The mechanism \ref{mech:duplex_HG_and FHG}, given the declared valuation $\mathbf{d}$, works as follows. Let the agents be ordered from $1$ to $n$ and consider agents in such ordering (line 1). Let $i$ be the agent being considered in the $i$-th {\em For} iteration (line 2). If there exists an agent $j$ such that $d_i(j)=1$ and $d_j(i)=1$ (line 2a) or $d_j(i)=0$ and $j$ is a sink (line 2b) or $d_j(i)=0$ and $j$ is an agent already considered by the mechanism, i.e., $j$ is before $i$ in the ordering (line 2c), then return agents $i$ and $j$ together in a coalition and each other agent in a coalition alone. If, after considering all the agents, the mechanism does not create the coalition with two agents, then return each agent in a coalition alone (line 3).

\begin{theorem}\label{UB:duplex:HG}
For ASHGs and FHGs with duplex valuations, Mechanism
\ref{mech:duplex_HG_and FHG} is a deterministic strategyproof
acceptable mechanism. The approximation ratio is $O(n^2)$ for ASHGs with duplex valuations, and 
$O(n)$ for FHGs with duplex valuations.
\end{theorem}
\begin{proof}

The mechanism \ref{mech:duplex_HG_and FHG} returns at most one coalition composed by two agents and all the other coalitions are composed by one agent alone. Moreover, no agent $i$ is put together with another agent $j$ in the same coalition if there is a value of $-1$ between them, that is if $d_i(j)=-1$ or $d_j(i)=-1$. This implies that no agent gets negative utility in the outcome returned by \ref{mech:duplex_HG_and FHG}, i.e., \ref{mech:duplex_HG_and FHG} is acceptable. More specifically, if a coalition of two agents is created, then such a coalition has positive (i.e., strictly greater than zero) social welfare. In particular, in ASHGs every agent gets utility $1$ or zero, while in FHGs $\frac{1}{2}$ or zero.   
Furthermore notice that, given the valuations declared by agents, if all the agents are sinks, then the optimal solution has social welfare zero and also \ref{mech:duplex_HG_and FHG} returns the outcome where each agent is in a coalition alone. On the other hand, if there is at least one agent that is not a sink, then it is not difficult to see that the optimal solution has positive social welfare. We now prove that, in such a case \ref{mech:duplex_HG_and FHG} would return a coalition with two agents together with positive social welfare. 

\begin{lemma}\label{lemma:if_opt_is_positive_then_also_mech_is_positive}
Given the valuations $\mathbf{d}=\langle d_1,...,d_n\rangle$ declared by agents, if there exists an agent $i$ that is not a sink, then Mechanism \ref{mech:duplex_HG_and FHG} returns an outcome where two agents are put together in the same coalition, thus yielding positive social welfare.   
\end{lemma}

\begin{proof}
First suppose that Mechanism \ref{mech:duplex_HG_and FHG} does not consider agent $i$ (line 2). It means that \ref{mech:duplex_HG_and FHG} has created a coalition with two agents before considering $i$ (with at least one agent of the coalition appearing before $i$ in the ordering). Suppose now that agent $i$ is considered by \ref{mech:duplex_HG_and FHG}. Then two scenarios are possible: i) agent $i$ is put together with another agent, still getting a positive social welfare, or ii) $i$ is put alone. This means that, for any agent $j$ such that $d_i(j)=1$ and $d_j(i) \neq -1$, agent $j$ is not a sink an she appears after $i$ in the ordering. Thus we can now consider the agent $j$ as the new one that is not a sink and apply the same argument as above.  Summarizing we have that at any step $s$ of the mechanism, if a coalition of two agents is not created, then there exists an agent that is not a sink and that is not considered at step $s$ yet. Therefore a coalition of two agents will be for sure created by \ref{mech:duplex_HG_and FHG} at some step after $s$.      
\end{proof}

We are now ready to show that Mechanism \ref{mech:duplex_HG_and FHG} is strategyproof. The following argument is valid for both ASHGs and FHGs. The proof relies on the analysis of different cases. 

Assume an agent $i$ gets positive (i.e., greater than zero) utility when she declares her valuations truthfully. Then, agent $i$ cannot improve her utility by declaring valuations $d_i \neq v_i$. In fact getting positive utility, that is utility $1$ or $\frac{1}{2}$ depending on whether we consider ASHGs or FHGs respectively, is the best she can obtain.

Assume now that an agent $i$ gets utility zero when she declares her valuations truthfully. We show that agent $i$ cannot improve her utility by declaring valuations $d_i \neq v_i$. If the agent $i$ is a sink then she has no incentive to lie. In fact, in this case $i$ would get positive utility only if she is put together an agent $j$ such that $d_i(j)=1$ and $d_j(i)=-1$. However the outcome returned by Mechanism \ref{mech:duplex_HG_and FHG} is such that no agent gets negative utility. Moreover $i$ has no incentive to declare a value of $1$ for some agent (in order to become not a sink anymore) if the real value is indeed different than $1$. It remains to consider the case where the agent $i$ is not a sink. By Lemma \ref{lemma:if_opt_is_positive_then_also_mech_is_positive} we know that in this case our mechanism always returns a coalition of two agents. Let us first suppose that such coalition, that we call $C_{j,z}$, is formed by agents $j$ and $z$ together. If $i$ has not been considered by the mechanism, that is, for instance the coalition $C_{j,z}$ has been created while considering agent $j$ that appears before $i$ in the ordering, then there is nothing that agent $i$ can do in order to get positive utility. Indeed the only thing that $i$ could do is (mis)-declaring $d_i(j)=1$ (if we suppose that $v_i(j) \neq 1$). In such a case, if also $d_j(i)=1$, the mechanism could return the coalition with $i$ and $j$ together. However agent $i$ would still not get positive utility. If $i$ has been considered by the mechanism but has not been put in a coalition together with another agent, then it means that while  \ref{mech:duplex_HG_and FHG} was considering agent $i$, for any $j$ such that $d_i(j)=1$, $j$ is not a sink and $j$ was not considered by the mechanism yet. We notice that $j$ has no incentive to declare a value of $1$ for some agents $z$ if the real value is not $1$ (i.e., $v_j(z) \neq 1$). Still there is nothing that $i$ can do. 

Let us finally suppose that the coalition of two agents returned by \ref{mech:duplex_HG_and FHG} contains agent $i$ (but still $i$ gets utility zero). This is only possible if, while mechanism \ref{mech:duplex_HG_and FHG} was considering agent $i$, it was not able to put $i$ together with another agent and (for the same reasons as in the previous case), there is nothing that agent $i$ can do to change it. In fact, agent $i$ could be put together another agent $j$, that appears after $i$ in the ordering, when the mechanism considers $j$. In this case the mechanism could put $i$ together with $j$ only if $d_j(i) = 1$. However it must be that $d_i(j) \neq 1$, otherwise the mechanism would have put $i$ and $j$ together while considering $i$, and therefore agent $i$ still does not get positive utility.

We now show the approximation ratio of the mechanism. If the optimal solution has social welfare zero, then also our mechanism returns an outcome (i.e., all the agents alone) with social welfare zero. If the optimal solution has positive social welfare (and thus there exists an agent that is not a sink), then by Lemma \ref{lemma:if_opt_is_positive_then_also_mech_is_positive}, we know that our mechanism returns an outcome with social welfare at least $1$ for ASHGs, and at least $\frac{1}{2}$ for FHGs. The theorem follows by noticing that, any agent can get utility at most $n-1$ for ASHGs and at most $1$ for FHGs. 
\end{proof} 

%We think it is worth to emphasize that, we tried to achieve a better approximation ratio by
%exploiting the idea used in Mechanism \ref{mech:duplex_HG_and FHG}. However, we realized that our analysis is tight in this sense. In fact, for instance, if we consider additively separable hedonic games, there exists an instance and an order of the agents for that instance, such that the optimal solution has value order of $n^2$, while \ref{mech:duplex_HG_and FHG} puts two agents in a coalition in the last iteration of the loop FOR. Thus, even if \ref{mech:duplex_HG_and FHG} does not terminate after putting two agents in a coalition, the analysis cannot be improved. Clearly, the mechanism could perform more loops FOR in order to match more than one pair of agents, however in such a case we can show that the mechanism is not strategyproof anymore. 

We point out that, if we consider ASHGs, there exists an instance and an ordering of the agents for that instance, such that the optimal solution has value order of $n^2$, while \ref{mech:duplex_HG_and FHG} puts two agents in a coalition in the last iteration of the loop {\em For}. Thus, even if \ref{mech:duplex_HG_and FHG} does not terminate after putting two agents in a coalition, the analysis cannot be improved. Clearly, \ref{mech:duplex_HG_and FHG} could perform more loops {\em For} in order to match more than one pair of agents. However, in such a case we can show that the mechanism is not strategyproof anymore. In fact, consider a cycle of 4 nodes with arcs $\{(1,2),(2,3),(3,4),(4,1)\}$, and all the weights $1$. The ordering is $1,2,3,4$. If the mechanism iterates the loop {\em For}, it would return in the first iteration agents $\{4,1\}$ in a coalition, and then, in a second iteration of the {\em For}, agents $\{2,3\}$ together. Notice that agent $1$ has utility zero. However, agent $1$ can improve her utility by declaring a further arc of weight $-1$ to agent $4$. In fact, in this case, in the first iteration the mechanism would put agents $\{3,4\}$ together, and then, in the second one, agents $\{1,2\}$.

%Concerning randomized strategyproof mechanisms for duplex valuations, we are only able to get a randomized mechanism with approximation ratio $O(n^2)$ for hedonic games. Clearly, our deterministic mechanism with the same approximation ratio is a stronger result.

\section{Simple valuations}\label{sec:Simple:valuations_results}

Exactly as in the case of non-negative valuations, for ASHGs with simple valuations, Mechanism $\mathcal{M}_1$
is acceptable and strategyproof and it also achieves the optimal social welfare. Therefore, we focus
on FHGs. We first prove that any deterministic
strategyproof mechanism cannot approximate better than
$\frac{6}{5}$ the social welfare. 
%even if we do not consider computational constraints.

\begin{theorem}\label{LB:NonNegative_and_Simple:FHG}
For FHGs with simple valuations, no
deterministic strategyproof acceptable mechanism has approximation
ratio less than $\frac{6}{5}$.
\end{theorem}

\begin{proof} Let us consider the instance $I_1$ depicted in
Figure~\ref{fig:LB_Simple_Non-negative_valuations_a}. The reader can
easily check (by considering all the possible coalitions) that an
optimal solution has social welfare $\frac{5}{3}$. It is composed
by the three coalitions where, two of them contain two consecutive
agents, and the remaining one contains three consecutive agents. For
instance, an optimal solution could be $C_1=\{1,2\}, C_2=\{3,4\},
C_3=\{5,6,7\}$. Notice that the grand coalition has social welfare
$1$. Therefore, a mechanism achieving an approximation better than
$\frac{5}{3}$, has to return more than one coalition. In such a
solution there always exists at least one agent, say agent $k$,
having utility zero. Let us now consider the instance $I_2$ depicted
in Figure~\ref{fig:LB_Simple_Non-negative_valuations_b}, where without loss of generality
we suppose that $k=2$. Again the reader can easily check (by
considering all the possible coalitions) that an optimal solution
has social welfare $2$. Such optimal solution is $C_1=\{2,3,4\},
C_2=\{5,6\}, C_3=\{1,7\}$. Once again the reader can check that any
solution where agents $2$ and $3$ are not in the same coalition,
(i.e., any solution where agent $2$ has utility equal to $0$ in the
instance $I_1$ ) can achieve a social welfare of at most
$\frac{5}{3}$, and therefore an approximation not better than
$\frac{6}{5}$. We conclude that any mechanism achieving an
approximation ratio strictly better than $\frac{6}{5}$, in both
instances $I_1$ and $I_2$, is not strategyproof.

\begin{figure}[h]%
    \centering
    %\subfigure
    \begin{subfigure}[b]{0.40\textwidth}
        \begin{tikzpicture}[scale=0.08,->,>=stealth',shorten >=1pt,auto,node distance=1.35cm,
                    thick,main node/.style={circle,draw,font=\sffamily\Large\bfseries}]
        %\tikzset{edge/.style = {->,> = latex'}};
        \node[main node] (1) {1};
        \node[main node] (2) [below right of=1] {2};
        \node[main node] (3) [below of=2] {3};
        \node[main node] (4) [below left of=3] {4};
        \node[main node] (5) [above left of=4] {5};
        \node[main node] (6) [above left of=5] {6};
        \node[main node] (7) [above right of=6] {7};
        \path[every node/.style={font=\sffamily\small}]
        (1) edge node {} (2)
        (2) edge node {} (3)
        (3) edge node {} (4)
        (4) edge node {} (5)
        (5) edge node {} (6)
        (6) edge node {} (7)
        (7) edge node {} (1);
        \end{tikzpicture}
        \caption{Instance $I_1$}
        \label{fig:LB_Simple_Non-negative_valuations_a}
    \end{subfigure}
    \begin{subfigure}[b]{0.40\textwidth}
        \begin{tikzpicture}[scale=0.08,->,>=stealth',shorten >=1pt,auto,node distance=1.35cm,
                    thick,main node/.style={circle,draw,font=\sffamily\Large\bfseries}]
        %\tikzset{edge/.style = {->,> = latex'}};
        \node[main node] (1) {1};
        \node[main node] (2) [below right of=1] {2};
        \node[main node] (3) [below of=2] {3};
        \node[main node] (4) [below left of=3] {4};
        \node[main node] (5) [above left of=4] {5};
        \node[main node] (6) [above left of=5] {6};
        \node[main node] (7) [above right of=6] {7};
        \path[every node/.style={font=\sffamily\small}]
        (1) edge node {} (2)
        (2) edge node {} (3)
        (3) edge node {} (4)
        (4) edge node {} (5)
        (5) edge node {} (6)
        (6) edge node {} (7)
        (7) edge node {} (1)
        (2) edge node {} (4);
        \end{tikzpicture}
        \caption{instance $I_2$}
        \label{fig:LB_Simple_Non-negative_valuations_b}
    \end{subfigure}

    \caption{The lower bound instance for simple valuations.}
    \label{fig:LB_Simple_Non-negative_valuations}
\end{figure}
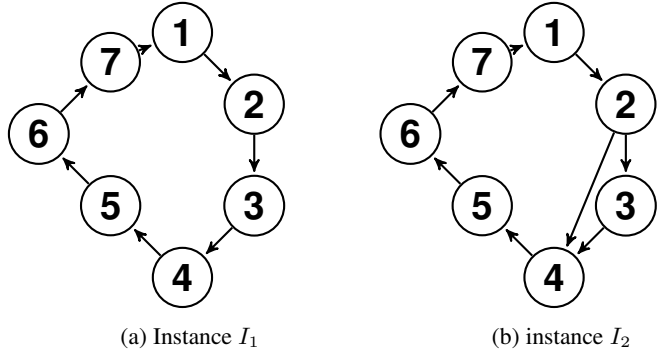
\end{proof}

We now show a deterministic strategyproof acceptable mechanism with
nearly optimal social welfare. Given the preferences declared by the agents $\mathbf{d}=d_1,\ldots,d_n$, and the associated directed
graph representation $G=(V,E)$ (notice that since we are considering
simple valuations, $d_i$ represents (indeed is) the set of arcs
outgoing from node $i$ in $G$), we construct an undirected weighted graph
$\bar{G}=(\bar{V},\bar{E})$, where $\bar{V}=V$. There is an
(undirected) edge $\{i,j\} \in \bar{E}$, if $(i,j) \in E$ or $(j,i)
\in E$. Finally, for each $\{i,j\} \in \bar{E}$, we have that
the weight $w(i,j)=1$ if either $(i,j) \in E$ or $(j,i) \in E$, and $w(i,j)=2$
if both $(i,j) \in E$ and $(j,i) \in E$ (otherwise $w(i,j)=0$, i.e.,
$\{i,j\} \notin \bar{E}$). A matching $m$ of $\bar{G}$ naturally
induces an outcome for fractional hedonic games, that is, any edge
$\{i,j\} \in m$ induces the coalition $C_{i,j}=\{i,j\}$, and for any
node $i$ not matched in $m$ we have the coalition $C_i=\{i\}$.
Notice that the coalitions induced by the matching are such that
each agent can have utility either $\frac{1}{2}$ or $0$. It is possible to show
that, finding the maximum weighted matching of
$\bar{G}=(\bar{N},\bar{E})$, using a consistent tie-breaking rule,
 gives a strategyproof mechanism.
 
The proof of the following lemma is
similar to the one proposed in \cite{DG2010}, which also shows that $\prec$-minimal matching can be found in poly-time.

\begin{lemma}\label{lemma:MaxMatching_strategyproof}
Fix a total order $\prec$ on matchings in the complete graph induced
by all the agents. %For a set of edges $E$, let $M(E)$ denote the set
%of matchings on edge set $E$.
 Let $\mathcal{M}$ be the mechanism
that, given the input $\mathbf{d}=\langle d_1,\ldots,d_n \rangle$, finds the
$\prec$-minimal matching $m$ on $\bar{G}=(\bar{V},\bar{E})$, such that $\sum_{\{i,j\} \in m} w(i,j)$ is maximized.
%in the set $argmax_{m \in M(\bar{E})}
%\sum_{\{i,j\} \in m} w(i,j)$,
%where $M(\bar{E})$ is the
%set of matching on edge $\bar{E}$.
Then $\mathcal{M}$ is strategyproof.
\end{lemma}

\begin{proof}
Assume for a contradiction that $\mathcal{M}$ is not truthful. Then
there exists $\bar{E}$ induced by edges $E_{-i} \cup E_i$, and $E'_i$
(inducing the edges set $\bar{E'}=\bar{E}_{-i} \cup \bar{E'}_i$),
violating the truthfulness. Let $m=\mathcal{M}(\bar{E})$ and
$m'=\mathcal{M}(\bar{E'})$. Agent $i$ has utility zero in the
coalitions induced by $m$, that is, for any $\{i,j\} \in m$ we have
that $(i,j) \notin E_i$. Yet agent $i$ has utility $\frac{1}{2}$ in
the coalitions induced by $m'$. It means that there exists $\{i,j\}
\in m'$ such that $(i,j) \in E_i$ (and then clearly $\{i,j\} \in
\bar{E}$). Moreover since the mechanism only uses declared edges and
agent $i$ has utility $\frac{1}{2}$ in the coalitions induced by
$m'$, it follows that there exists $\{i,j\} \in m'$ such that $(i,j)
\in E'_i \cap E_i$. It implies that both $m$ and $m'$ are in
$M(\bar{E}) \cap M(\bar{E'})$. Since the mechanism returns the
maximum matching it follows that $m$ and $m'$ are optimal in both
$M(\bar{E})$ and $M(\bar{E'})$. Recalling that $\mathcal{M}$ breaks
ties consistently, this yields a contradiction, as needed.
\end{proof}

%The proof (in the Appendix) uses a
%similar idea to the one proposed in \cite{DG2010}, which also shows that $\prec$-minimal matching can be found in poly-time.
%%By using ideas contained in \cite{DG2010} it is possible to show
%%that the maximum matching that is $\prec$-minimal can be done in
%%polynomial time.

Now we prove the approximation ratio of the mechanism. Given an undirected
graph $G=(V,E)$, where $w$ is the edges weight function,  we denote
by $w(E)$ the sum of the weights of the edges belonging to $E$,
i.e., $w(E)=\sum_{\{i,j\} \in E} w(i,j)$.

\begin{theorem}\label{thm:upper bound max matching}
The deterministic mechanism outputting the maximum matching as
described in Lemma~\ref{lemma:MaxMatching_strategyproof} is
strategyproof and acceptable with approximation ratio of $2$.
\end{theorem}

\begin{proof}
Let $m$ be the matching computed by the mechanism and
$\mathcal{C}^m$ be the coalitions induced by $m$. Let
$\mathcal{C}^*=\{C^*_1,\ldots,C^*_p\}$ be optimal coalitions (we do not consider optimal coalitions having social welfare equal to zero, indeed we can ignore them). Let
$m'=m'_1 \cup \ldots \cup m'_p$ where $m'_h$, $1\leq h \leq p$, is a
maximum matching in the graph induced by the vertices of $C^*_h$.
Let $\mathcal{C}^{m'}$ be the coalitions induced by $m'$. Let $A_h$
be the vertices matched in $m'_h$ and $B_h = C^*_h \setminus A_h$.
Notice that $B_h$ is a stable set and that $|A_h|$ is an even
number.

\begin{proposition}\label{claim:bound on the edges_matching}
When $|B_h|>0$, then for any $h=1,\ldots,p$, and any edge $\{i,j\}
\in m'_h$, we have that $\sum_{b \in B_h} w(i,b) + w(j,b) \leq
w(i,j)(|B_h|+1)$.
\end{proposition}

\begin{proof}
First notice that, for any $b \in B_h$, it holds that $w(i,b) \leq w(i,j)$ and
$w(j,b) \leq w(i,j)$, since otherwise we can get a better matching by
removing the edge $\{i,j\}$ from $m'_h$ and adding the new edge
having weight strictly greater than $w(i,j)$. We now distinguish two
cases depending on the size of $B_h$. If $|B_h|>1$, then suppose that
$\sum_{b \in B_h} w(i,b) + w(j,b) > w(i,j)(|B_h|+1)$. It implies
there are two distinct edges $\{i,b\}$ and $\{j,b'\}$ for some $b,b'
\in B_h$ such that $w(i,b)+w(j,b') > w(i,j)$ and then contradicting
the fact that $m'_h$ is a maximum matching in $C^*_h$. If $|B_h|=1$
then the claim easily follows from the observation that $w(i,b) \leq
w(i,j)$ and $w(j,b) \leq w(i,j)$.
\end{proof}

Let $\hat{E_h}$ be the set of edges of the graph induced by the
vertices of $A_h$ minus the edges belonging to the matching $m'_h$.
Moreover, let $w(\hat{E_h}) = \sum_{\{i,j\} \in \hat{E_h}} w(i,j)$.

\begin{proposition}\label{claim:ratio_between_matching_and_weight_of_the_graph}
For any $h=1,\ldots,p$, then $w(\hat{E_h}) \leq w(m'_h)(|A_h|-2)$.
\end{proposition}

\begin{proof}
Assume for a contradiction that $w(\hat{E_h}) > w(m'_h)(|A_h|-2)$. Let us consider the graph $G_{A_h}$ induced by the vertices of $A_h$ and suppose that $G_{A_h}$ is complete (if it is not complete, we can just add edges of weights zero). It is easy to see that all the edges of such complete graph can be partitioned into $|A_h| - 1$ different perfect matchings (recall that $|A_h|$ is an even number). It implies that must exist a perfect matching in $G_{A_h}$ having weight at least equal to $\frac{w(\hat{E_h}) + w(m'_h)}{|A_h|-1} > \frac{w(m'_h)(|A_h|-2) + w(m'_h)}{|A_h|-1} = w(m'_h)$ thus contradicting the fact that $m'_h$ is a maximum matching.
\end{proof}

Then, when $|B_h|>0$, by using propositions~\ref{claim:bound on the
edges_matching}
and~\ref{claim:ratio_between_matching_and_weight_of_the_graph}, we
can bound the social welfare of $C^*_h$, for any $h=1,\ldots,p$, as
follows:

%\vskip -1.5em

\begin{align}
& SW(C^*_h) = \nonumber \\
& = \frac{1}{|C^*_h|}[\sum_{\{i,j\} \in m'_h}
(w(i,j)+\sum_{b \in B_h} w(i,b) + w(j,b)) + w(\hat{E_h})]   \nonumber \\
& \leq \frac{1}{|C^*_h|}[w(m'_h) + w(m'_h)(|B_h| + 1)
+ w(m'_h)(|A_h|-2)]\nonumber \\
& = w(m'_h) . \nonumber
\end{align}
Similarly, when $|B_h|=0$ we can get that $SW(C^*_h) \leq w(m'_h)$.
Therefore, overall we have that $SW(C^*) \leq w(m')$. Since it is
easy to see that $w(m) \geq w(m')$, then we have that the social
welfare of $\mathcal{C}^m$ is
\begin{align}
SW(\mathcal{C}^m) = \frac{w(m)}{2} \geq \frac{w(m')}{2} \geq
\frac{SW(\mathcal{C}^*)}{2}. \nonumber
\end{align}  %\qed

\end{proof}

%We point out that, when dealing with FHGs, it is natural to resort on matchings. Many papers (for instance \cite{ABH2014,AGGMT2015,BFFMM14,BFFMM15}) used them. The challenge is how to exploit their properties, and in this sense we make some steps forward. Indeed, we better exploit properties of maximum weighted matchings. This is proved by the fact that, in the next Subsection \ref{subsec:extensions}, we also show that our analysis can be used to improve from the $4$-approximation of maximum weighted matching for symmetric valuations (i.e., undirected graph) (Theorem 7 of the paper \cite{AGGMT2015}), to a $2$-approximation. Another remark is that, our results are not only working for the approximation of asymmetric fractional hedonic games, i.e., directed graphs, but also include the strategyproofness, which was not considered before for fractional hedonic games.

We point out that, when dealing with FHGs, it is natural to resort on matchings. Many papers (for instance \cite{ABH2014,AGGMT2015,BFFMM14,BFFMM15}) used them. The challenge is how to exploit their properties, and in this sense we make some steps forward. Indeed, we better exploit properties of maximum weighted matchings. This is proved by the fact that, our analysis can be used to improve from the $4$-approximation (Theorem 7 of the paper \cite{AGGMT2015}) of maximum weighted matching for symmetric valuations, i.e., undirected graph, to a $2$-approximation. Another remark is that, our results are not only working for the approximation of asymmetric FHGs, i.e., directed graphs, but also include the strategyproofness, which was not considered before for FHGs.

\iffalse
\subsection{Extension}\label{subsec:extensions}

In this section, we do not consider strategyproof mechanisms. Here, we want to emphasize that, 
by using techniques very similar to Theorem~\ref{thm:upper
bound max matching}, it is possible to prove that, given a weighted and undirected graph inducing
an instance of the fractional hedonic game with additive symmetric
valuation function (in the symmetric game, for any couple of agents $i$ and $j$,
it holds that $v_i(j)=v_j(i)$), the outcome
induced by a maximum weighted matching is a $2$-approximation of the
maximum social welfare. This improves the result in \cite{AGGMT2015},
where they prove a bound of $4$.

\begin{theorem}\label{thm:improved_apx_symmetric_no_mechanism}
For the symmetric fractional hedonic games, the coalitions induced
by the maximum weighted matching (that can be computed in polynomial
time) is a $2$-approximation of the maximum social welfare.
\end{theorem}
\fi

%Computing the maximum social welfare in symmetric fractional hedonic games is an NP-hard problem \cite{AGGMT2015}, and we are not aware of stronger inapproximability results. Therefore, the $2$-approximation algorithm of Theorem \ref{thm:improved_apx_symmetric_no_mechanism} is not tight in this sense.
We finally notice that, $2$ is the best approximation achievable by using matchings, when dealing with the problem of computing the maximum social welfare in symmetric fractional hedonic games. In fact, consider a complete graph of $n$ nodes. In the grand coalition, each node has utility $\frac{n-1}{n}$ (consider big $n$), while in a matching, each node has utility at most $\frac{1}{2}$.

\section{Conclusion and future work}\label{Sec:Conclusion_and_future_work}
In this paper, we studied strategyproof mechanisms for ASHGs and FHGs, under general and specific additive
valuation functions. Despite the theoretical
interest for specific valuations, for which we were able to show
better bounds with respect to generic valuations, specific
valuations also model realistic scenarios. 

%For instance, simple
%valuations model a basic economic scenario referred to in the literature as
%Bakers and Millers \cite{ABH2014,BFFMM15}. 
%
%We proved some differences between ASHGs and FHGs, and shed
%some light on the capabilities and limitations of strategyproof
%mechanisms for both kind of games. 
%We were mainly interested in deterministic mechanisms, however in a few cases we have also provided randomized lower bounds.

Our paper leaves some appealing open problems. First of all, it would be nice to close the gaps of Table \ref{table}, and in particular the gap of deterministic strategyproof mechanisms for duplex valuations. Moreover, it is worth to understand whether randomized strategyproof mechanisms can achieve significantly better performance than deterministic ones. It would be also important to understand what happens when valuations are drawn at random from some distribution (in order to avoid the bad instances), or when there are size constraints to the coalitions. Finally, another research direction, is that of considering more general valuation functions than additive ones.

%\begin{itemize}
%\item First of all, it would be nice to close the gaps of Table \ref{table}, and in particular the gap of deterministic strategyproof mechanisms for duplex valuations. 
%
%\item Moreover, it is worth to understand whether randomized strategyproof mechanisms can achieve significantly better performance than deterministic ones. 
% 
%\item It would be also important to understand what happens when valuations are drawn at random from some distribution (in order to avoid the bad instances), or when there are size constraints to the coalitions.
%
%\item Finally, another research direction, is that of considering more general valuation functions than additive ones.
%\end{itemize}

%\newpage
%\bibliographystyle{named}
%\bibliography{ref}

\end{document}